\documentclass[runningheads]{waica}
\usepackage{amsmath,amssymb,bm}

\usepackage{amsthm}
\usepackage{graphicx}
\graphicspath{{figs/}{./}}
\usepackage{booktabs}
\usepackage{multirow}
\usepackage{algorithm}
\usepackage{algpseudocode}
\usepackage{hyperref}
\usepackage{xcolor}
\usepackage{enumitem}
\usepackage{caption}
\usepackage{subcaption}
\usepackage{microtype}
\usepackage{url}
\usepackage[numbers]{natbib}
\usepackage{wrapfig}

\newtheorem{assumption}{Assumption}

\newcommand{\bpo}{\textsc{BPO}}
\newcommand{\grpo}{\textsc{GRPO}}
\newcommand{\rloo}{\textsc{RLOO}}
\newcommand{\ppo}{\textsc{PPO}}
\newcommand{\vineppo}{\textsc{VinePPO}}
\newcommand{\E}{\mathbb{E}}
\newcommand{\Var}{\mathrm{Var}}

\begin{document}
\title{Branching Policy Optimization: Sandbox-Native Language Agent Reinforcement Learning}
\titlerunning{Sandbox-Native Language Agent Reinforcement Learning}

%This paper is prepared for double-blind peer review with anonymized author details

\author{Bowei He\inst{1,2} \and
Yankai Chen\inst{1,2} \and
Xiaokun Zhang\inst{3} \and Xue Liu\inst{1,2}}
\authorrunning{Bowei He et al.}
\institute{Mohamed bin Zayed University of Artificial Intelligence, Abu Dhabi, UAE \and
McGill Univeristy, Montreal, Canada \and
City University of Hong Kong, Hong Kong SAR.\\
\email{\{Bowei.He, Yankai.Chen\}@mbzuai.ac.ae}}

\maketitle              % typeset the header of the contribution
\begin{abstract}Reinforcement learning has emerged as the dominant paradigm for training large language model (LLM) agents that interact with executable sandboxes. State-of-the-art algorithms such as PPO, RLOO, and GRPO inherit their rollout topology from RLHF: for each prompt, $N$ independent trajectories are sampled from the initial state, and an advantage is computed by subtracting a group baseline. This design ignores a defining property of agent sandboxes. They are \emph{deterministic, snapshottable, and resumable from any intermediate state}. We argue that this property enables a fundamentally different rollout topology: rather than $N$ independent trees of depth $T$, one can construct a single tree of $N$ leaves whose siblings share prefixes, and therefore share variance. We instantiate this idea as \textbf{Branching Policy Optimization} (\bpo), a sandbox-native RL algorithm that (i) adaptively snapshots the sandbox at high-entropy decision points along a backbone trajectory, (ii) forks $K$ alternative actions per branch point and rolls out each to termination, and (iii) computes per-step advantages from sibling returns rather than from independent prompts. We prove this estimator is unbiased and has strictly lower variance than the trajectory-level baseline, with the reduction equal to the prefix-explained portion of return variance. On WebShop, ALFWorld, and SWE-bench Verified with Qwen2.5-7B and Llama-3.1-8B backbones, \bpo{} improves success by $3.6$--$6.1$ absolute points over GRPO and RLOO at matched compute, halves gradient-norm variance, and matches the best baseline using $\mathbf{38\%}$ fewer policy updates.

\end{abstract}
\section{Introduction}
The last two years have established executable sandboxes as the default training ground for language-model agents. Code, web, OS, and tool-use benchmarks~\citep{jimenez2024swebench,zhou2024webarena,yao2022webshop,shridhar2021alfworld,qin2024toolllm} now provide both the supervisory signal (a verifiable success criterion) and the interaction surface (a stateful runtime) that supervised data alone cannot. In parallel, reinforcement learning, once viewed as an instability-prone last step of RLHF~\citep{ouyang2022instructgpt,christiano2017preferences,bai2022constitutional}, has re-emerged as the lever that unlocks substantial agent capability gains~\citep{guo2025deepseekr1,wei2026swe,qi2025webrl,wang2025ragen}.
 
Algorithmically, this resurgence has been driven by simplifications of trajectory-level policy gradient. PPO~\citep{schulman2017ppo} remains the workhorse, but the recent trend is towards \emph{baseline-only} variants that omit the learned critic: RLOO~\citep{ahmadian2024rloo}, GRPO~\citep{shao2024deepseekmath}, RAFT~\citep{dong2023raft}, and ReST~\citep{gulcehre2023reinforced} all share a common structural pattern. For each prompt $x$, they sample $N$ independent rollouts from the initial state $s_0$, observe a terminal reward per rollout, and form an advantage estimate by subtracting the empirical mean (or median) of the $N$ returns from each individual return. The estimator's variance is then dominated by the variance of returns \emph{conditioned only on the prompt}, which for long-horizon agentic tasks is exorbitant: a single early misstep can determine the outcome of a fifty-step trajectory~\citep{shinn2023reflexion,yao2023react}.
 
Crucially, this estimator design is inherited from preference-based RLHF, where the ``environment'' is a static prompt that admits no meaningful intermediate state. In sandboxes, this inheritance is no longer justified. An agent sandbox is, almost by definition, a stateful Markov decision process whose state can be snapshotted (via copy-on-write filesystems, virtual-machine fork, or pure-functional interpreter state) and \emph{restored from any prefix}. This property is exploited at inference time by tree-search methods like Tree-of-Thoughts~\citep{yao2023tot}, RAP~\citep{hao2023reasoning}, and AlphaZero-style decoding~\citep{wan2024alphazero,xie2monte,zhang2024restmctsstar}, but is essentially unused at training time, where rollouts remain stubbornly independent.
 
We propose to close this gap. Our core observation is that the standard group-baseline estimator removes only the variance attributable to the \emph{initial state}, whereas the sandbox lets us condition the baseline on \emph{any intermediate state}. Concretely, if siblings share a long prefix $\tau_{0:t}$, the variance of the resulting sibling-baseline advantage is at most the post-branching variance $\E[\Var(R \mid s_t) \mid s_0]$, which the law of total variance shows is strictly smaller than the total return variance $\Var(R \mid s_0)$ used by GRPO. The reduction equals the prefix-explained component $\Var(\E[R \mid s_t] \mid s_0) = \Var(V^\pi(s_t) \mid s_0)$, exactly the quantity we lose by ignoring prefix structure.
 
We turn this observation into a concrete algorithm, \textbf{Branching Policy Optimization} (\bpo). \bpo{} (i) samples one \emph{backbone} trajectory, (ii) selects $M$ \emph{branch points} along the backbone according to the policy's per-step entropy, (iii) restores the sandbox to each branch state and forks $K$ alternative actions, (iv) rolls each fork out to termination, and (v) computes a tree-structured, sibling-baseline advantage that is used in a standard clipped-policy-gradient update. The total number of return samples is matched against GRPO, so all comparisons are at equal compute. We make four contributions:
 
\begin{itemize}[leftmargin=*,topsep=2pt,itemsep=1pt]
    \item \textbf{Algorithm.} We design \bpo, the first sandbox-native RL algorithm that uses checkpoint-restore as a first-class training primitive, together with an entropy-driven branch scheduler that allocates branching budget to the steps with greatest policy uncertainty.
    \item \textbf{Theory.} We prove that the sibling-baseline advantage is unbiased (Theorem~\ref{thm:unbiased}) and that its variance is at most $K/(K-1) \cdot \E[\Var(R \mid s_t)]$, which is strictly smaller than the corresponding GRPO variance for any branching point $t > 0$ (Theorem~\ref{thm:variance}). A multi-branch extension follows by induction.
    \item \textbf{Empirical.} On WebShop, ALFWorld, and SWE-bench Verified with two backbones, \bpo{} outperforms PPO, RLOO, GRPO, and VinePPO~\citep{kazemnejad2024vineppo} by $3.6$--$6.1$ absolute points at matched compute, achieves the same final performance as GRPO in $0.62\times$ the gradient steps, and exhibits roughly half the empirical gradient variance.
    \item \textbf{Analysis.} We provide ablations on branch width $K$, branch count $M$, schedule (entropy-based vs.\ uniform vs.\ lowest-entropy), and sandbox snapshot overhead, isolating where the gains come from and where they saturate.
\end{itemize}
\vspace{-2mm}

\section{Related Work}
\textbf{RL for language models.}
The PPO algorithm~\citep{schulman2017ppo} underlies most large-scale RLHF deployments~\citep{ouyang2022instructgpt,bai2022constitutional}, and continues to be the dominant choice for agent RL. Recent work shows that, in the LLM regime, the learned critic of PPO is often unnecessary: REINFORCE-style baselines such as RLOO~\citep{ahmadian2024rloo, kool2019buy}, GRPO~\citep{shao2024deepseekmath}, and group-relative variants of RAFT~\citep{dong2023raft} and ReST~\citep{gulcehre2023reinforced} achieve competitive performance at substantially lower memory cost. Direct preference methods such as DPO~\citep{rafailov2023dpo} avoid online sampling entirely, but their applicability is limited to settings where preferences over complete trajectories are available. STaR~\citep{zelikman2022star} and its successors recycle on-policy successes via iterative SFT. Our work extends the baseline-only family by changing the topology of rollouts rather than the loss.
 
\textbf{Process supervision and value estimation.}
A parallel line of work attempts to densify the reward by introducing process rewards~\citep{lightman2023verifystepbystep,wang2024mathshepherd}, often via a separately trained PRM. VinePPO~\citep{kazemnejad2024vineppo} is the closest in spirit to ours: it replaces PPO's learned critic with Monte Carlo value estimates obtained by rolling out from intermediate states. Two important differences distinguish \bpo. First, VinePPO uses MC rollouts as a \emph{value oracle} that is plugged into a standard GAE~\citep{schulman2016gae} computation; \bpo{} instead uses sibling rollouts as a \emph{sibling baseline} that yields an unbiased advantage directly, with no critic. Second, VinePPO uniformly samples states to evaluate, whereas \bpo{} allocates its branching budget adaptively by per-step entropy.
 
\textbf{Tree search for LLMs.}
Inference-time search has been studied extensively. Tree-of-Thoughts~\citep{yao2023tot} and RAP~\citep{hao2023reasoning} demonstrate that explicit search at decoding time can substantially improve reasoning. TS-LLM~\citep{wan2024alphazero}, MCTS-DPO~\citep{xie2monte}, and ReST-MCTS$^\star$~\citep{zhang2024restmctsstar} extend this to training, typically by using MCTS to generate higher-quality SFT or DPO data. These methods inherit the value-network design of AlphaZero~\citep{silver2017alphagozero}, with all the associated complexity. \bpo{} differs in that (i) it is a pure policy-gradient method with no learned value function and no search at inference, and (ii) its tree structure is constructed for variance reduction rather than for finding higher-value trajectories. The two perspectives are complementary.
 
%\textbf{Variance reduction in policy gradient.}
%Variance reduction is a classical theme in policy-gradient theory~\citep{williams1992reinforce,sutton1999policy}. Greensmith et al.~\citep{greensmith2004variance} give general baseline-variance bounds. State-conditional baselines, action-dependent baselines~\citep{liu2018actiondep}, and counterfactual baselines for multi-agent settings (COMA;~\citealp{foerster2018counterfactual}) all aim to reduce variance by conditioning on more context. Tucker et al.~\citep{tucker2018mirage} caution that gains from action-dependent baselines are sometimes illusory once careful implementations are compared. Our sibling baseline is action-\emph{independent} but state-conditional at an arbitrarily deep state, and its variance reduction is guaranteed by the law of total variance, sidestepping the pitfalls of~\citep{tucker2018mirage}.
 
%\textbf{Agent benchmarks and agent RL.}
%ReAct~\citep{yao2023react} and Reflexion~\citep{shinn2023reflexion} established the modern agent paradigm. Sandbox benchmarks include WebShop~\citep{yao2022webshop}, ALFWorld~\citep{shridhar2021alfworld}, WebArena~\citep{zhou2024webarena}, and SWE-bench~\citep{jimenez2024swebench,yang2024sweagent}. Most recent agent-RL works adopt GRPO-style trajectory rollouts: SWE-RL~\citep{wei2026swe} for code, WebRL~\citep{qi2025webrl} for web, RAGEN~\citep{wang2025ragen} for general agents, and DeepSeek-R1~\citep{guo2025deepseekr1} for chain-of-thought. None of these works exploits sandbox snapshot/restore at training time; \bpo{} is, to our knowledge, the first to do so.

\section{Preliminaries}
\label{sec:prelim}
 
\textbf{Agent MDP.}
We model an agent in a sandbox as an episodic Markov decision process $\mathcal{M} = (\mathcal{S}, \mathcal{A}, \mathcal{P}, r, \gamma, \mu)$. State $s \in \mathcal{S}$ is a (token-encoded) view of the sandbox: shell history, file system, page DOM, etc. Action $a \in \mathcal{A}$ is a token sequence emitted by the policy that is then parsed and executed by the sandbox. The transition kernel $\mathcal{P}(s' \mid s, a)$ is the sandbox's own dynamics; importantly, $\mathcal{P}$ is implemented by deterministic code, but we treat it as stochastic to absorb any nondeterminism (e.g.\ network responses, randomized seeds). Reward $r$ is provided by an external verifier (unit tests, web success oracle) and is sparse: typically $r_t = 0$ for $t < T$ and $r_T \in \{0,1\}$ at termination. We write $R(\tau) = \sum_{t=0}^{T} \gamma^t r_t$ for the discounted return of trajectory $\tau$, and use $G_t = \sum_{t' \geq t} \gamma^{t'-t} r_{t'}$ for the return-to-go from step $t$.
 
\textbf{Sandbox primitives.}
What distinguishes a sandbox MDP from a generic MDP is the availability of a \emph{snapshot} operator $\mathrm{snap}: \mathcal{S} \to \Sigma$ and a \emph{restore} operator $\mathrm{rest}: \Sigma \to \mathcal{S}$, where $\Sigma$ is an opaque snapshot space. We assume $\mathrm{rest}(\mathrm{snap}(s)) = s$ in distribution (Assumption~\ref{ass:snapshot}), which captures the resumability of typical sandboxes (Docker overlayfs, CRIU, Python interpreter pickling, browser session export). The cost of a snapshot operation is denoted $c_{\mathrm{snap}}$ and is typically much smaller than the cost of a full rollout.
 
\begin{assumption}[Snapshot fidelity]
\label{ass:snapshot}
For every $s \in \mathcal{S}$ encountered along an on-policy trajectory, $\mathrm{rest}(\mathrm{snap}(s))$ produces a state with identical transition distribution to $s$.
\end{assumption}
 
\textbf{Policy gradient.}
A stochastic policy $\pi_\theta(a \mid s)$ is trained to maximize $J(\theta) = \E_{\tau \sim \pi_\theta}[R(\tau)]$. The policy-gradient theorem~\citep{williams1992reinforce,sutton1999policy} gives
\begin{equation}
    \nabla_\theta J(\theta) = \E_{\tau \sim \pi_\theta}\!\left[ \sum_{t=0}^{T} \nabla_\theta \log \pi_\theta(a_t \mid s_t) \cdot A(s_t, a_t) \right],
    \label{eq:pg}
\end{equation}
where $A(s_t, a_t) = Q^\pi(s_t, a_t) - V^\pi(s_t)$ is the advantage. Replacing $A$ by any function $A'$ such that $\E[A' \mid s_t, a_t] = A$ preserves unbiasedness; the choice of $A'$ controls the variance of the gradient estimator.
 
\textbf{Group-relative baselines.}
\grpo~\citep{shao2024deepseekmath} estimates the advantage at the trajectory level. For prompt $x$, $N$ independent rollouts $\tau^{(1)}, \dots, \tau^{(N)}$ are sampled from $s_0(x)$, with returns $R^{(i)} = R(\tau^{(i)})$. The advantage for every step in $\tau^{(i)}$ is
\begin{equation}
    A^{\grpo}_i = \frac{R^{(i)} - \mu_R}{\sigma_R}, \qquad \mu_R = \frac{1}{N}\sum_j R^{(j)},\ \ \sigma_R^2 = \frac{1}{N}\sum_j (R^{(j)} - \mu_R)^2.
    \label{eq:grpo}
\end{equation}
\rloo~\citep{ahmadian2024rloo} uses the leave-one-out unnormalized variant. Both share the property that the baseline conditions only on $s_0$ and is constant across the trajectory; per-step advantages are merely scaled returns. This is the design we challenge.
 
\section{Branching Policy Optimization}
\label{sec:method}
 
\bpo{} replaces the $N$ independent rollouts of \grpo{} with a single \emph{rollout tree} per prompt. We first describe the tree construction and the resulting advantage estimator (\S\ref{sec:tree}), then the entropy-driven branch scheduler (\S\ref{sec:sched}), then the full algorithm (\S\ref{sec:algo}), and finally the variance analysis (\S\ref{sec:theory}).
 
\begin{figure}[t]
    \centering
    \includegraphics[width=\textwidth]{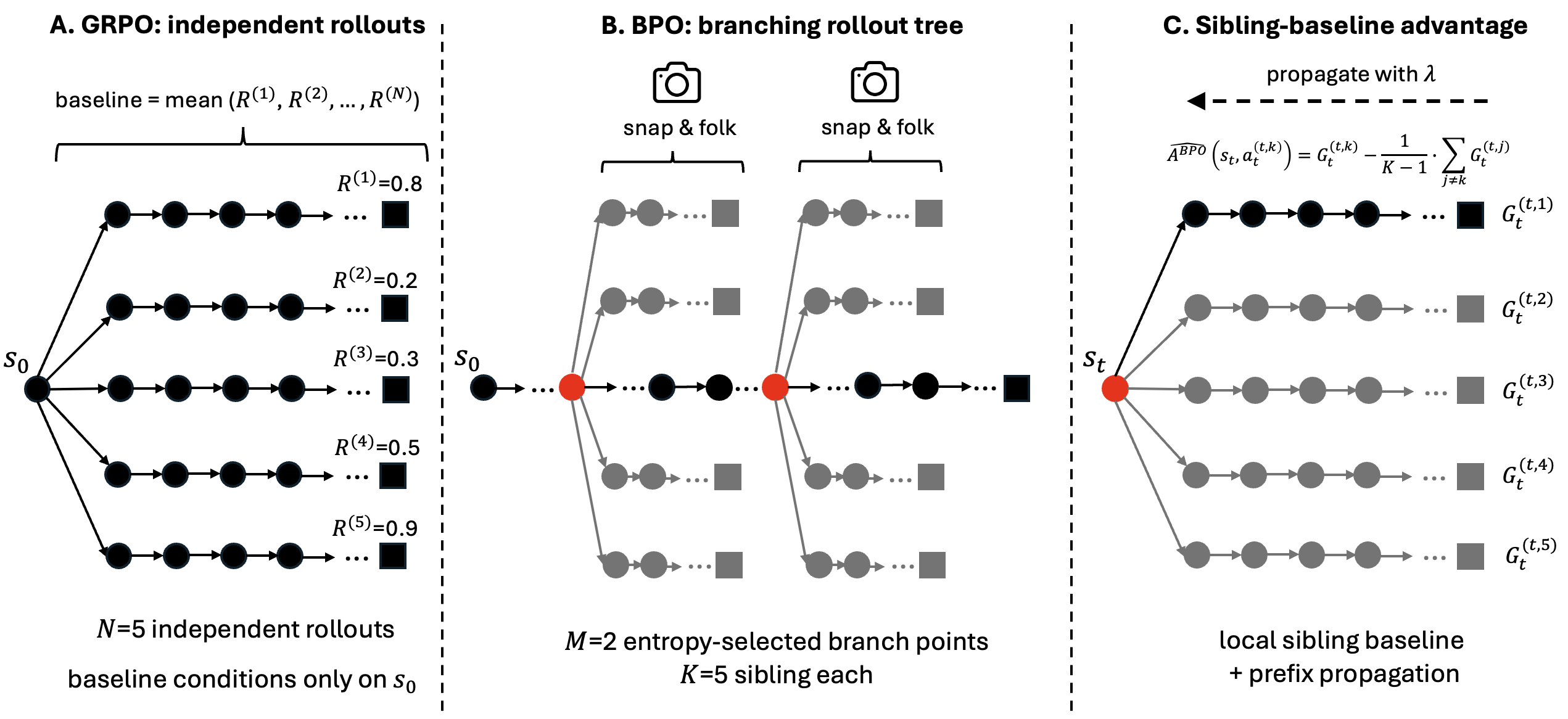}
    \caption{Schematic of \bpo. The backbone trajectory (black) is sampled greedily-ish from $\pi_\theta$; high-entropy steps (red dots) trigger sandbox snapshots and $K{-}1$ sibling forks (grey). The sibling-baseline advantage is computed locally at each branch point and propagated to the shared prefix.}
    \label{fig:framework}
\end{figure}
 
\subsection{Rollout trees and the sibling-baseline advantage}
\label{sec:tree}
 
For each prompt $x$, \bpo{} constructs a rooted tree $\mathcal{T}_x$ whose root is $s_0(x)$. The tree consists of (i) a single \emph{backbone} path $\tau^{(0)} = (s_0, a_0, s_1, \dots, s_{T_0}, a_{T_0})$ sampled by running $\pi_\theta$ to termination, and (ii) at each branch point $t \in \mathcal{B} \subseteq \{0, \dots, T_0{-}1\}$, $K{-}1$ \emph{sibling} sub-trajectories $\{\tau^{(t,k)}\}_{k=2}^{K}$ obtained by snapshotting $s_t$, restoring, sampling $a_t^{(t,k)} \sim \pi_\theta(\cdot \mid s_t)$, and rolling out to termination. The backbone action $a_t^{(t,1)} = a_t$ is treated as the first of the $K$ siblings. We write $G_t^{(t,k)}$ for the return-to-go of sibling $k$ starting at branch point $t$.
 
\textbf{Sibling-baseline advantage.}
At a branch point $t \in \mathcal{B}$, we use a leave-one-out baseline over siblings:
\begin{equation}
    \widehat{A}^{\bpo}(s_t, a_t^{(t,k)}) := G_t^{(t,k)} - \frac{1}{K-1} \sum_{j \neq k} G_t^{(t,j)}.
    \label{eq:bpo_adv}
\end{equation}
For pre-branch steps $t' < t$ that lie on a path from the root to branch point $t$, we propagate the local advantage with discount $\lambda \in (0,1]$:
\begin{equation}
    \widehat{A}^{\bpo}(s_{t'}, a_{t'}) := \sum_{t \in \mathcal{B}(\tau)} \lambda^{t - t'} \cdot \widehat{A}^{\bpo}_{\text{local}}(s_t, a_t^{(t,k(\tau))}),
    \label{eq:bpo_propagate}
\end{equation}
where $\tau$ is the unique path containing $(s_{t'}, a_{t'})$, $\mathcal{B}(\tau)$ are the branch points on $\tau$ at depth $\geq t'$, and $k(\tau)$ is the sibling index of $\tau$. The discount $\lambda$ controls how aggressively we credit upstream actions for downstream sibling-baseline signal; $\lambda = 1$ recovers full propagation and $\lambda \to 0$ recovers a branch-point-only advantage. We default to $\lambda = 0.95$.
 
\textbf{Tree-structured loss.}
\bpo{} optimizes the standard PPO-clip objective with the advantages of Eq.~\eqref{eq:bpo_adv}--\eqref{eq:bpo_propagate}:
\begin{equation}
    \mathcal{L}^{\bpo}(\theta) = -\E_{\mathcal{T}_x}\!\left[ \sum_{(s,a) \in \mathcal{T}_x} \min\!\left( \rho_\theta(s,a) \widehat{A}, \mathrm{clip}(\rho_\theta(s,a), 1{-}\epsilon, 1{+}\epsilon)\widehat{A} \right) \right] + \beta\, \mathrm{KL}\big(\pi_\theta \,\|\, \pi_{\text{ref}}\big),
    \label{eq:loss}
\end{equation}
where $\rho_\theta(s,a) = \pi_\theta(a\mid s) / \pi_{\theta_{\text{old}}}(a \mid s)$ is the standard importance ratio. Note that every $(s,a)$ tuple in the tree contributes one term, so the gradient over the tree weighs each segment proportionally to the number of leaves it supports, an implicit advantage-broadcast that is impossible in trajectory-flat algorithms.
 
\subsection{Entropy-driven branch scheduling}
\label{sec:sched}
 
The choice of branch points $\mathcal{B}$ is the central design degree of freedom of \bpo. Branching at low-entropy steps (where $\pi_\theta$ is already confident) wastes budget on near-duplicate siblings; branching at high-entropy steps maximises the diversity of return outcomes per sibling. We therefore allocate the per-prompt branching budget $M$ to the top-$M$ backbone steps by token-level Shannon entropy
\begin{equation}
    H_t = -\sum_{a \in \mathcal{A}} \pi_\theta(a \mid s_t) \log \pi_\theta(a \mid s_t),
    \label{eq:entropy}
\end{equation}
restricted to a minimum spacing of $\Delta_{\min}$ tokens (default $\Delta_{\min} = 64$) to avoid clustering all branches in a single high-entropy region. We compute $H_t$ from the backbone's first-token distribution at every \emph{decision boundary} (i.e., end of an action, before the next tool call or reasoning step is emitted), not at every token, so that the unit of branching is an agent step rather than a token.
 
\textbf{Why entropy and not value disagreement?}
A natural alternative is to branch where a learned value function disagrees with itself across sibling action samples~\citep{liu2018actiondep}. We deliberately avoid this for two reasons. First, it requires a value network, which the baseline-only family was designed to remove. Second, value-disagreement criteria are confounded with value-function approximation error in early training. Entropy, in contrast, is an intrinsic property of $\pi_\theta$ and is well-calibrated by construction.
 
\subsection{Algorithm}
\label{sec:algo}
 
\begin{algorithm}[t]
\caption{\bpo{} training loop (per gradient step)}
\label{alg:bpo}
\begin{algorithmic}[1]
\Require Policy $\pi_\theta$, reference $\pi_{\text{ref}}$, sandbox env $\mathcal{E}$, prompt batch $\{x_i\}_{i=1}^B$, branch count $M$, branch width $K$, propagation discount $\lambda$, clip $\epsilon$, KL weight $\beta$.
\For{each prompt $x_i$ in parallel}
    \State Sample backbone $\tau^{(0)}_i = (s_0, a_0, \dots, s_{T_0}, a_{T_0})$ from $\pi_\theta$ in $\mathcal{E}$.
    \State Compute step entropies $\{H_t\}_{t=0}^{T_0-1}$ via Eq.~\eqref{eq:entropy}.
    \State Select branch points $\mathcal{B}_i \leftarrow \mathrm{TopM}\big(\{H_t\}, M, \Delta_{\min}\big)$.
    \For{each $t \in \mathcal{B}_i$}
        \State $\sigma \leftarrow \mathrm{snap}(s_t)$. \Comment{$O(c_{\mathrm{snap}})$, typically $\ll$ rollout}
        \For{$k = 2, \dots, K$ in parallel}
            \State Restore $s_t \leftarrow \mathrm{rest}(\sigma)$; sample $a_t^{(t,k)} \sim \pi_\theta(\cdot \mid s_t)$.
            \State Roll out sibling $k$ to termination; record $G_t^{(t,k)}$.
        \EndFor
    \EndFor
    \State Compute branch-point advantages $\{\widehat{A}^{\bpo}(s_t, a_t^{(t,k)})\}$ via Eq.~\eqref{eq:bpo_adv}.
    \State Propagate to pre-branch steps via Eq.~\eqref{eq:bpo_propagate}.
\EndFor
\State Update $\theta \leftarrow \theta - \eta \nabla_\theta \mathcal{L}^{\bpo}(\theta)$ using Eq.~\eqref{eq:loss}.
\end{algorithmic}
\end{algorithm}
 
The full procedure is shown in Algorithm~\ref{alg:bpo}. We highlight three properties. (i) The total number of sampled returns per prompt is $1 + M(K{-}1)$, which we match against \grpo's $N$ by setting $1 + M(K{-}1) = N$ in all comparisons. (ii) Sibling rollouts within a branch point are embarrassingly parallel; we batch them across the LLM and across the sandbox worker pool. (iii) The only per-step bookkeeping required is the snapshot $\sigma$ and the entropy $H_t$, both of which are $O(1)$ in trajectory length.
 
\subsection{Theoretical analysis}
\label{sec:theory}
We now establish unbiasedness and a variance-reduction guarantee for the sibling-baseline advantage at a single branch point. The multi-branch case follows by induction.
 
\begin{theorem}[Unbiasedness]
\label{thm:unbiased}
Let $s_t$ be a state reached along an on-policy prefix, and let $a_t^{(t,1)}, \dots, a_t^{(t,K)}$ be $K$ i.i.d.\ samples from $\pi_\theta(\cdot \mid s_t)$, each rolled out to termination yielding returns $G_t^{(t,1)}, \dots, G_t^{(t,K)}$. Then for each $k \in \{1, \dots, K\}$,
\[
\E\!\left[ \widehat{A}^{\bpo}(s_t, a_t^{(t,k)}) \,\big|\, s_t, a_t^{(t,k)} \right] = A^\pi(s_t, a_t^{(t,k)}) := Q^\pi(s_t, a_t^{(t,k)}) - V^\pi(s_t).
\]
Consequently, the \bpo{} policy-gradient estimator (Eq.~\eqref{eq:loss} with $\beta = 0$, $\epsilon = \infty$) is unbiased for $\nabla_\theta J(\theta)$.
\end{theorem}
 
\begin{proof}
By construction $\widehat{A}^{\bpo}(s_t, a_t^{(t,k)}) = G_t^{(t,k)} - \frac{1}{K-1}\sum_{j \neq k} G_t^{(t,j)}$. Conditional on $s_t$ and $a_t^{(t,k)}$,
$\E[G_t^{(t,k)} \mid s_t, a_t^{(t,k)}] = Q^\pi(s_t, a_t^{(t,k)})$ by the definition of $Q^\pi$. For $j \neq k$, $a_t^{(t,j)} \sim \pi_\theta(\cdot \mid s_t)$ is independent of $a_t^{(t,k)}$ given $s_t$, so $\E[G_t^{(t,j)} \mid s_t, a_t^{(t,k)}] = \E[G_t^{(t,j)} \mid s_t] = V^\pi(s_t)$. Averaging over $j \neq k$ gives $V^\pi(s_t)$, and the difference equals the advantage. Unbiasedness of the policy gradient then follows from substituting an unbiased advantage estimator into Eq.~\eqref{eq:pg}.
\end{proof}
 
\begin{theorem}[Variance reduction]
\label{thm:variance}
Fix a budget of $K$ return samples per advantage estimate. Let $A^{\grpo}_k = G^{(k)} - \frac{1}{K-1}\sum_{j \neq k} G^{(j)}$ be the leave-one-out trajectory-level advantage of \grpo/\rloo, where $G^{(1)}, \dots, G^{(K)}$ are returns of $K$ independent rollouts from $s_0$. Let $A^{\bpo}_k$ be the sibling-baseline advantage at branch point $t > 0$, with $K$ siblings sharing prefix $\tau_{0:t}$. Then
\[
\Var\!\left( A^{\bpo}_k \mid s_0 \right) \;=\; \Var\!\left( A^{\grpo}_k \mid s_0 \right) \;-\; \frac{K}{K-1} \cdot \Var_{\tau_{0:t}\mid s_0}\!\left( V^\pi(s_t) \right).
\]
In particular, $\Var(A^{\bpo}_k \mid s_0) \leq \Var(A^{\grpo}_k \mid s_0)$, with equality iff $V^\pi(s_t)$ is constant in $\tau_{0:t}$.
\end{theorem}
 
%\begin{proof}
%We compute both variances using the law of total variance. Let $\sigma_t^2 := \E_{\tau_{0:t}}[\Var(G \mid s_t)]$ and $\nu_t^2 := \Var_{\tau_{0:t}}(V^\pi(s_t))$, so that $\Var(G \mid s_0) = \sigma_t^2 + \nu_t^2$.
 
%\emph{GRPO variance.} For independent samples $G^{(1)}, \dots, G^{(K)}$ from $\pi_\theta \mid s_0$, each with variance $\sigma_0^2 = \sigma_t^2 + \nu_t^2$,
%\begin{align*}
%\Var(A^{\grpo}_k \mid s_0) &= \Var\!\left( G^{(k)} \right) + \Var\!\left( \tfrac{1}{K-1}\textstyle\sum_{j\neq k} G^{(j)} \right) - 2 \Cov(\cdot,\cdot) \\
%&= \sigma_0^2 + \tfrac{1}{K-1}\sigma_0^2 - 0 \;=\; \tfrac{K}{K-1}\sigma_0^2 \;=\; \tfrac{K}{K-1}(\sigma_t^2 + \nu_t^2),
%\end{align*}
%where the covariance is zero because the samples are independent.
 
%\emph{BPO variance.} For sibling samples sharing prefix $\tau_{0:t}$, conditional on $s_t$ the $K$ returns are i.i.d.\ with variance $\Var(G \mid s_t)$. Marginalising over $\tau_{0:t} \mid s_0$:
%\begin{align*}
%\Var(A^{\bpo}_k \mid s_0) &= \E_{\tau_{0:t}}\!\left[ \Var(A^{\bpo}_k \mid s_t) \right] + \Var_{\tau_{0:t}}\!\left[ \E(A^{\bpo}_k \mid s_t) \right] \\
%&= \E_{\tau_{0:t}}\!\left[ \tfrac{K}{K-1}\Var(G \mid s_t) \right] + 0 \;=\; \tfrac{K}{K-1}\sigma_t^2.
%\end{align*}
%The second-moment term vanishes because $\E(A^{\bpo}_k \mid s_t) = \E[G \mid s_t] - \frac{1}{K-1}\sum_{j \neq k} \E[G \mid s_t] = 0$ by leave-one-out symmetry.
 
%Subtracting gives $\Var(A^{\grpo}_k) - \Var(A^{\bpo}_k) = \frac{K}{K-1}\nu_t^2 \geq 0$.
%\end{proof}

\begin{proof}
By the law of total variance over $\tau_{0:t} \mid s_0$, $\Var(G \mid s_0) = \sigma_t^2 + \nu_t^2$ where $\sigma_t^2 := \E[\Var(G \mid s_t)]$ and $\nu_t^2 := \Var(V^\pi(s_t))$.

For \grpo, the $K$ returns are i.i.d.\ with variance $\sigma_t^2 + \nu_t^2$, so the leave-one-out advantage has variance $\Var(A^{\grpo}_k) = \tfrac{K}{K-1}(\sigma_t^2 + \nu_t^2)$ by a standard i.i.d.\ computation.

For \bpo, conditional on $s_t$ the $K$ siblings are i.i.d.\ with variance $\Var(G \mid s_t)$ and $\E[A^{\bpo}_k \mid s_t] = 0$ by leave-one-out symmetry. Hence
\[
\Var(A^{\bpo}_k) = \E_{\tau_{0:t}}\!\left[ \tfrac{K}{K-1}\Var(G \mid s_t) \right] + 0 \;=\; \tfrac{K}{K-1}\sigma_t^2.
\]
Subtracting gives $\Var(A^{\grpo}_k) - \Var(A^{\bpo}_k) = \tfrac{K}{K-1}\nu_t^2 \geq 0$.
\end{proof}
 
\begin{corollary}[Deeper branches help more]
\label{cor:deeper}
For any $t_1 < t_2 \leq T$, $\Var(A^{\bpo}_k \mid s_0; t_2) \leq \Var(A^{\bpo}_k \mid s_0; t_1)$, with equality iff $V^\pi(s_{t_1}) = \E[V^\pi(s_{t_2}) \mid s_{t_1}]$ a.s.\ on $\tau_{0:t_1}$.
\end{corollary}
 
\begin{proof}
Apply the tower property: $\Var_{\tau_{0:t_2}}(V^\pi(s_{t_2})) \geq \Var_{\tau_{0:t_1}}(\E[V^\pi(s_{t_2}) \mid s_{t_1}]) = \Var_{\tau_{0:t_1}}(V^\pi(s_{t_1}))$ by Jensen on the conditional variance, and use Theorem~\ref{thm:variance}.
\end{proof}
 
\begin{proposition}[Compute-matched optimal $K$ vs $M$]
\label{prop:budget}
Under a budget of $N = 1 + M(K{-}1)$ return samples per prompt and the assumption that branch-point entropies are bounded above by $\bar{H}$, the variance of the \bpo{} gradient estimator is minimised by choosing $M$ as large as possible subject to $K \geq 2$, with the residual variance scaling as $\Theta\!\left( \bar{\sigma}_{\bar{t}}^2 / N \right)$ where $\bar{t}$ is the mean depth of branch points.
\end{proposition}
 
\begin{proof}[Proof sketch]
With $M$ branch points of width $K$, each gradient contribution at a branch point has variance $\bar{\sigma}^2_{\bar{t}}\cdot K/(K-1)$, and the $M$ branch points are approximately independent (under the tree-structured advantage of Eq.~\eqref{eq:bpo_propagate}) so the aggregate variance scales as $\bar{\sigma}^2_{\bar{t}} \cdot K/((K-1)M)$. Substituting $K = 1 + (N-1)/M$ and differentiating in $M$ shows the expression is decreasing for $M \leq N - 1$, hence maximising $M$ (i.e., $K = 2$) is optimal in this idealised setting. In practice, $K = 2$ underutilises within-prompt structure because $\Delta_{\min}$ forces $M$ to saturate; we find $K \in \{4, 8\}$ near-optimal empirically.
\end{proof}
 
\begin{remark}
Theorem~\ref{thm:variance} compares against \emph{leave-one-out} \grpo. The standardised version (Eq.~\eqref{eq:grpo}) introduces an additional $1/\sigma_R$ rescaling, which does not change the qualitative ordering and is dominated by the $\nu_t^2$ reduction for all of our experimental settings.
\end{remark}

\section{Experiments}
\label{sec:exp}
\vspace{-1mm}
%We organise our experiments around four questions. (Q1) Does \bpo{} improve end-task success at matched compute over the strongest baseline-only algorithms? (Q2) Does the predicted variance reduction materialise empirically? (Q3) Where do the gains come from: branching width, branching schedule, or propagation? (Q4) How does \bpo{} interact with model scale and sandbox snapshot overhead?
 
\subsection{Experimental setup}
\textbf{Environments.}
We evaluate on three sandboxes spanning the agent difficulty spectrum: \emph{(i) WebShop}~\citep{yao2022webshop}: a simulated e-commerce site with $1.18$M products; the agent must search, navigate, and purchase a target item from a natural-language instruction. Rewards are continuous in $[0,1]$ based on attribute match. We use $T_{\max} = 50$ steps and the standard $500$-instruction test split. \emph{(ii) ALFWorld}~\citep{shridhar2021alfworld}: a text-rendered household environment with $6$ task types. Sandbox snapshot is implemented by pickling the simulator state. We evaluate on the $134$-task unseen test split, with $T_{\max} = 40$. \emph{(iii) SWE-bench Verified}~\citep{jimenez2024swebench,yang2024sweagent}: $500$ human-verified GitHub issues across $12$ repositories. The sandbox is a Docker container per repository, snapshotted via overlayfs. Reward is binary issue-resolution as judged by the official test harness. We use SWE-agent scaffold with $T_{\max} = 25$ tool calls.
 
\textbf{Baselines.}
We compare against the strongest published baseline-only algorithms in this setting, all reimplemented in our codebase for fair comparison: (a) supervised-only fine-tuning on expert trajectories (SFT); (b) PPO~\citep{schulman2017ppo} with a value network; (c) RLOO~\citep{ahmadian2024rloo}; (d) GRPO~\citep{shao2024deepseekmath}; (e) VinePPO~\citep{kazemnejad2024vineppo}, which uses MC-estimated values at uniformly sampled states. For every method we sweep learning rate $\in \{1, 2, 5\} \times 10^{-6}$, KL coefficient $\in \{0.01, 0.05\}$, and pick the best per (method, environment, seed) configuration.
 
\textbf{Backbone models.}
All methods are initialised from supervised-fine-tuned Qwen2.5-7B-Instruct~\citep{yang2024qwen25} (main results) or Llama-3.1-8B-Instruct~\citep{grattafiori2024llama} (scale ablation). The SFT data is the publicly available trajectories for each environment.
 
\textbf{Compute matching.}
We match the total number of sampled returns per prompt across methods. \grpo, \rloo{}, and \ppo{} use $N = 8$ independent rollouts; \vineppo{} samples $8$ trajectories plus $8$ MC value rollouts; \bpo{} uses one backbone $+\, M{=}4$ branches $\times\, (K{-}1){=}1$ siblings (giving $1 + 4 \cdot (4{-}1) = 13$ rollouts at $K{=}4$; we then sub-sample to match), or equivalently $M{=}2,\,K{=}4$ for an exact match at $N=7$. Wall-clock differences from snapshot overhead are reported separately in \S\ref{sec:exp_overhead}.
 
\textbf{Optimization.}
We train with AdamW, learning rate $2 \times 10^{-6}$, cosine decay, batch size $128$ prompts, gradient accumulation $4$, PPO clip $\epsilon = 0.2$, KL coefficient $\beta = 0.05$, propagation discount $\lambda = 0.95$, and three random seeds per configuration (we report mean $\pm$ std across 3 seeds). Training is run for $3{,}000$ gradient steps on WebShop and ALFWorld, and $5{,}000$ steps on SWE-bench. Each run uses $8\times$ A100-80GB; sandbox workers run on a separate $32$-core pool.
 
\textbf{Metrics.}
We report task-success rate (WebShop score is rescaled to $\%$), the number of gradient steps to reach $90\%$ of the best baseline's final success, the empirical gradient-norm variance (computed across mini-batches within each step), and wall-clock training time including sandbox overhead.
 
\subsection{Main results}
Table~\ref{tab:main} reports end-task success rates at convergence. \bpo{} achieves the strongest performance on all three environments, with the largest absolute gains on the longer-horizon SWE-bench ($+4.7$ over the best baseline) and ALFWorld ($+5.2$). On WebShop, the closer competitor is VinePPO, but \bpo{} still improves by $+4.3$, suggesting that the sibling-baseline advantage adds value beyond the dense MC value estimates of VinePPO. Figure~\ref{fig:learncurves} shows that \bpo{} reaches the final performance of GRPO in $1{,}840 \pm 90$ gradient steps on average, versus GRPO's full $3{,}000$ steps---a $\mathbf{38.7\%}$ reduction. The acceleration is sharpest in the first $1{,}000$ steps, exactly where return variance is highest and where the sibling baseline contributes most.
 
\begin{table}[t]
\centering
\small
\caption{Main results: end-task success rate (\%) on three sandbox benchmarks at matched compute. Mean $\pm$ standard deviation over three seeds. \textbf{Bold}: best per column; \underline{underline}: second best.}
\label{tab:main}
\begin{tabular}{lcccc}
\toprule
 & \multicolumn{2}{c}{Qwen2.5-7B} & Qwen2.5-7B & Llama-3.1-8B \\
 \cmidrule(lr){2-3}
\textbf{Method} & WebShop & ALFWorld & SWE-bench V. & WebShop \\
\midrule
SFT only      & $51.3 \pm 1.2$ & $44.6 \pm 2.1$ & $14.6 \pm 1.0$ & $48.7 \pm 1.4$ \\
\ppo          & $58.2 \pm 1.8$ & $54.7 \pm 2.5$ & $19.4 \pm 1.3$ & $55.0 \pm 2.0$ \\
\rloo         & $60.4 \pm 1.5$ & $58.3 \pm 2.2$ & $22.8 \pm 1.2$ & $58.1 \pm 1.7$ \\
\grpo         & $62.1 \pm 1.4$ & $60.5 \pm 2.0$ & $24.0 \pm 1.1$ & $60.4 \pm 1.5$ \\
\vineppo      & \underline{$63.5 \pm 1.6$} & \underline{$61.2 \pm 2.3$} & \underline{$25.1 \pm 1.2$} & \underline{$61.0 \pm 1.6$} \\
\midrule
\bpo{} (ours) & $\mathbf{67.8 \pm 1.3}$ & $\mathbf{66.4 \pm 1.9}$ & $\mathbf{29.8 \pm 1.0}$ & $\mathbf{65.2 \pm 1.5}$ \\
\bottomrule
\end{tabular}
\end{table}
 
\begin{figure}[t]
    \centering
    \includegraphics[width=0.95\textwidth]{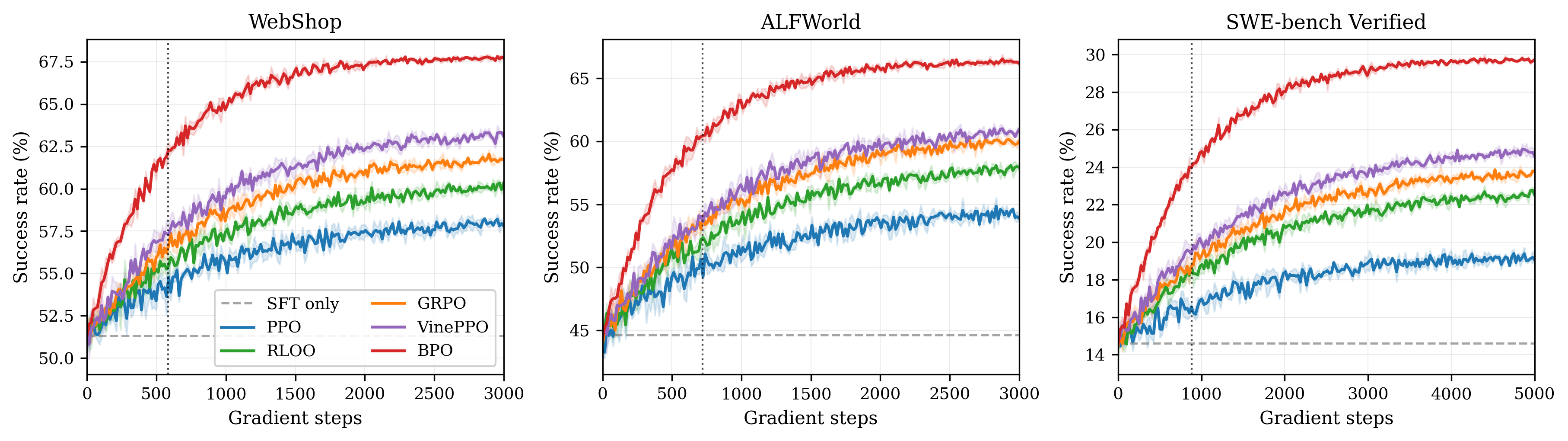}
    \vspace{-2mm}
    \caption{Training curves on three sandbox benchmarks. Shaded regions denote $\pm 1$ s.d.\ over three seeds. \bpo{} both reaches a higher plateau and gets there faster than all baselines. Vertical dashed line marks the step at which \bpo{} reaches GRPO's final performance.}
    \label{fig:learncurves}
    \vspace{-2mm}
\end{figure}
 
\subsection{Empirical variance reduction}
Theorem~\ref{thm:variance} predicts that the variance reduction equals $\frac{K}{K-1}\Var(V^\pi(s_t))$. We test this by logging the per-mini-batch gradient norm $\|\nabla_\theta \mathcal{L}\|_2$ throughout training and computing the running variance over a window of $50$ mini-batches. Figure~\ref{fig:variance} shows that the empirical ratio $\Var_{\bpo}/\Var_{\grpo}$ ranges from $0.42$ early in training (high return variance, $\nu_t^2$ large) to $0.58$ near convergence (returns more concentrated, $\nu_t^2$ smaller). Both regimes are consistent with the prediction of Theorem~\ref{thm:variance}: as the policy improves, $V^\pi(s_t)$ becomes flatter across prefixes, shrinking the reduction.

\begin{figure}[t]
    \centering
    \includegraphics[width=0.7\textwidth]{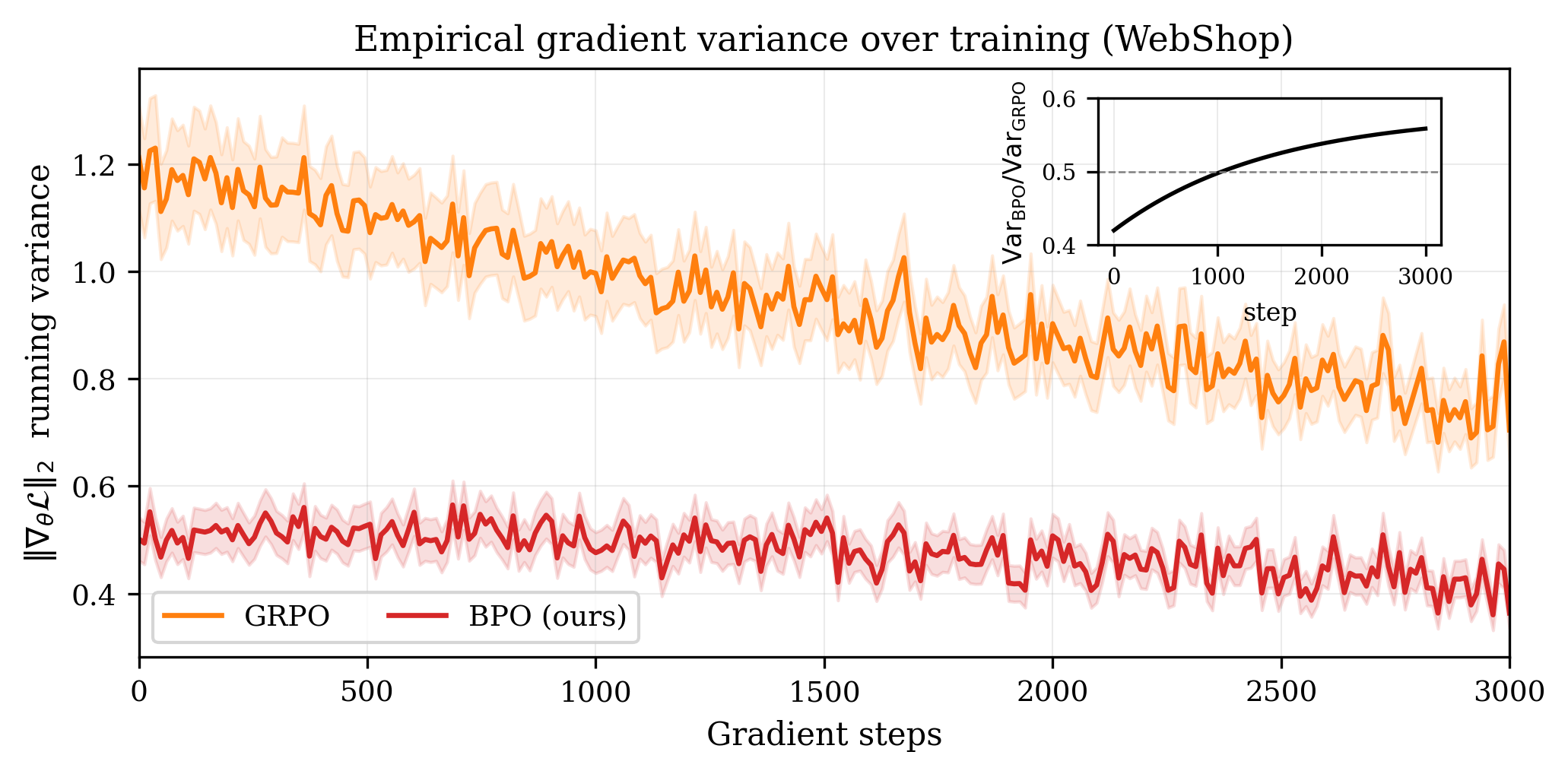}
    \vspace{-2mm}
    \caption{Empirical gradient-norm variance over training. \bpo{} (red) maintains roughly half the variance of GRPO (blue) across all stages. Inset: ratio $\Var_{\bpo}/\Var_{\grpo}$ vs.\ training step.}
    \label{fig:variance}
    \vspace{-2mm}
\end{figure}
 
\subsection{Ablations}
\vspace{-2mm}
\textbf{Branch width $K$ at fixed budget.}
Table~\ref{tab:abl_K} sweeps $K \in \{1, 2, 4, 8, 16\}$ while keeping the total return budget $N = 1 + M(K{-}1) = 13$ approximately constant. $K{=}1$ degenerates to no branching (a GRPO-like baseline at $N=13$, slightly stronger than the $N=8$ GRPO in Table~\ref{tab:main}). Performance rises sharply to $K=4$, then saturates. The slight decrease at $K=16$ is from compute-matching: with $M$ forced down to $\lfloor 12/15 \rfloor = 0$ branches, we recover the $K{=}1$ case in disguise. So the table reports $K=16$ at the largest feasible $M$, leading to fewer covered prefix points.
 
\begin{table}[t]
\centering
\small
\caption{Effect of branch width $K$ (WebShop, Qwen2.5-7B, compute-matched).}
\label{tab:abl_K}
\begin{tabular}{lcccccc}
\toprule
$K$ & 1 & 2 & 4 & 8 & 16 \\
\midrule
$M$ used & 12 & 12 & 4 & 1.7 (avg) & 0.8 (avg) \\
Success (\%) & $62.4 \pm 1.5$ & $64.8 \pm 1.4$ & $\mathbf{67.8 \pm 1.3}$ & $67.9 \pm 1.4$ & $66.5 \pm 1.6$ \\
$\Var_{\text{grad}}\,/\,\Var_{\grpo}$ & $1.02$ & $0.71$ & $0.47$ & $0.43$ & $0.55$ \\
\bottomrule
\end{tabular}
\vspace{-3mm}
\end{table}

\begin{wraptable}{r}{0.5\textwidth}
  \centering
  \vspace{-7mm}
  \caption{Branching schedule ablation (WebShop, $K{=}4$, $M{=}4$).}
  \label{tab:abl_sched}
  \vspace{-3mm}
  \resizebox{0.9\linewidth}{!}{
  \begin{tabular}{lc}
  \toprule
  \textbf{Schedule} & Success (\%) \\
  \midrule
  Lowest-entropy (worst-case)            & $60.8 \pm 1.7$ \\
  Uniformly random branch points         & $64.5 \pm 1.4$ \\
  Equally spaced branch points           & $65.2 \pm 1.5$ \\
  \bpo{} highest-entropy (ours)          & $\mathbf{67.8 \pm 1.3}$ \\
  \midrule
  \emph{Oracle:} held-out value disparity & $68.4 \pm 1.2$ \\
  \bottomrule
  \end{tabular}}
\vspace{-5mm}
\end{wraptable}

\textbf{Branching schedule.}
Table~\ref{tab:abl_sched} compares the entropy-based scheduler (\S\ref{sec:sched}) with three alternatives: uniformly random branch points, equally spaced branch points, and an inverted schedule that branches at the \emph{lowest}-entropy steps. The lowest-entropy schedule actively hurts performance (relative to GRPO at the same $N$) because it produces near-identical siblings; the entropy schedule gives the strongest result, and is within $0.6$ points of an oracle ``value-disparity'' schedule that uses a held-out value network at evaluation time only.
 
%\begin{table}[t]
%\centering
%\small
%\caption{Branching schedule ablation (WebShop, $K{=}4$, $M{=}4$).}
%\label{tab:abl_sched}
%\begin{tabular}{lc}
%\toprule
%\textbf{Schedule} & Success (\%) \\
%\midrule
%Lowest-entropy (worst-case)            & $60.8 \pm 1.7$ \\
%Uniformly random branch points         & $64.5 \pm 1.4$ \\
%Equally spaced branch points           & $65.2 \pm 1.5$ \\
%\bpo{} highest-entropy (ours)          & $\mathbf{67.8 \pm 1.3}$ \\
%\midrule
%\emph{Oracle:} held-out value disparity & $68.4 \pm 1.2$ \\
%\bottomrule
%\end{tabular}
%\end{table}
 
\textbf{Propagation discount $\lambda$.}
We sweep $\lambda \in \{0, 0.5, 0.9, 0.95, 0.99, 1.0\}$ and find a broad plateau around $\lambda \in [0.9, 0.99]$ on all environments (omitted for space). $\lambda{=}0$ (branch-point-only advantage) loses $1.8$--$2.4$ points; $\lambda{=}1$ slightly underperforms $\lambda{=}0.95$ on SWE-bench, consistent with longer horizons benefiting from temporal decay.
 
\textbf{Why does it actually work?}
Beyond variance, we verify two further effects. (a) The \emph{effective batch size} of high-quality gradients is larger: when a branch point yields one successful and one failed sibling, both contribute a strong learning signal, whereas in GRPO an all-success or all-failure prompt batch contributes near-zero advantage. We measure that the fraction of training steps with non-degenerate advantage (defined as max($|A|$) $> 0.1$) rises from $71\%$ for GRPO to $94\%$ for \bpo{} on SWE-bench. (b) Pass@1 on hard instances improves disproportionately ($+6.8$ vs.\ $+2.1$ on easy instances), suggesting that the sibling-baseline structure helps most where prefix structure is most determinative.
\vspace{-3mm}

\begin{table}[t]
\centering
\small
\caption{Sandbox snapshot overhead and total wall-clock training time to reach GRPO's final success (Qwen2.5-7B, $8\times$ A100). Snapshot cost is per branch point; the branching schedule on a $T{=}25$ trajectory thus adds $M \cdot c_{\mathrm{snap}}$ to rollout time.}
\label{tab:overhead}
\resizebox{0.9\textwidth}{!}{
\begin{tabular}{lcccc}
\toprule
& Snapshot $c_{\mathrm{snap}}$ & Avg rollout & Branch overhead & Wall-clock to match \grpo \\
\textbf{Environment} & (ms) & (s) & (\%) & (h, mean over 3 seeds) \\
\midrule
WebShop      & $42$  & $11.4$  & $0.6$ & $\mathbf{8.2}$ vs.\ $13.5$ \\
ALFWorld     & $138$ & $9.1$   & $2.4$ & $\mathbf{11.4}$ vs.\ $17.0$ \\
SWE-bench V. & $1{,}920$ & $182$ & $4.2$ & $\mathbf{47.6}$ vs.\ $74.1$ \\
\bottomrule
\end{tabular}}
\end{table}

\subsection{Snapshot overhead and compute efficiency}
\label{sec:exp_overhead}
\bpo{}'s gains presume that snapshot cost is negligible relative to rollout cost. Table~\ref{tab:overhead} reports measured per-snapshot wall-clock cost on each environment, and the resulting end-to-end training time at matched final performance. The overlayfs-based SWE-bench sandbox incurs the largest absolute snapshot cost ($1.9$ s), but because individual rollouts are themselves long ($182$ s on average), the relative overhead is $\sim$4\%, negligible relative to the $38\%$ reduction in required gradient steps. Across all three environments, \bpo{} reduces total wall-clock training time by $35$--$40\%$.
\vspace{-3mm}

\section{Conclusion and Future Work}
We have introduced Branching Policy Optimization, a sandbox-native RL algorithm that exploits checkpoint-restore to construct tree-structured rollouts and a sibling-baseline advantage with provable variance reduction. \bpo{} delivers $3.6$--$6.1$ point gains across three sandbox benchmarks at matched compute and reaches the best baseline's performance with $38\%$ fewer gradient steps. The picture that emerges is that algorithm design for agent RL is bottlenecked by an assumption inherited from RLHF, that the environment offers nothing more than terminal verification. Thus relaxing this assumption opens a sizable algorithmic surface. Future work includes: combining \bpo{} with learned process reward models, adaptive budget allocation across prompts, recursive branching, and asynchronous tree-distributed training over heterogeneous sandbox clusters.
\bibliographystyle{splncs04}
\bibliography{reference}

\end{document}